\documentclass[12pt]{article}

\RequirePackage{kpfonts}
\RequirePackage[sf,bf,small,raggedright]{titlesec}
\usepackage[a4paper, margin=2cm]{geometry}

\usepackage{dsfont}
\usepackage{comment}
\usepackage{textcomp}
\usepackage{amssymb}
\usepackage{bbm}
\usepackage{fancyhdr}
\usepackage{amsmath}
\usepackage{amsbsy,amsthm}
\usepackage{amscd}
\usepackage{latexsym}
\usepackage{graphicx}   % to include and manipulate graphics
\usepackage{pdfsync}
\usepackage{blkarray}
\usepackage{multirow}
\usepackage{hyperref}
\usepackage{xcolor}
\usepackage{enumerate}

\usepackage{tcolorbox}

\usepackage{tikz}

\usepackage{subcaption}
\usepackage{algorithmic}
\usepackage[justification=centering]{caption}

\usepackage{natbib}
 \bibpunct[, ]{[}{]}{,}{n}{}{,}%

\usepackage{float}
\usepackage{pgf}
\usetikzlibrary{arrows}

\usepackage{enumitem}

\usepackage{fancyhdr}
\newtheorem{theorem}{Theorem}[section]
\newtheorem{corollary}[theorem]{Corollary}
\newtheorem{lemma}[theorem]{Lemma}

\newtheorem{definition}[theorem]{Definition}

\newtheorem{def/prop}[theorem]{Definition/Proposition}

\newcommand{\R}{\mathbf{R}}

\newcommand{\prob}[1]{\mathbf{P}\left( #1 \right)}
\newcommand{\NormOp}[1]{\left\vert\kern-0.25ex\left\vert\kern-0.25ex\left\vert #1 \right\vert\kern-0.25ex\right\vert\kern-0.25ex\right\vert}

\newcommand{\cov}{\mathrm{Cov}}
\newcommand{\cip}{\mbox{\,$\perp\!\!\!\perp$\,}}
\newcommand{\indep}{\cip}

\DeclareMathOperator{\polylog}{polylog} 

\newcommand{\diam}[1]{\mathrm{diam}\!\left(#1\right)}
\newcommand{\maxdeg}[1]{\max\!\deg\left(#1\right)}

\newcommand{\degree}{\partial}

\title{Testing properties of trees in graphical models with covariance queries}

\author{Sofiya Burova 
\thanks{
IMTECH and Departament de Matem\`atiques,  Universitat Polit\`ecnica de Catalunya, Barcelona, Spain 
and 
Department of Economics and Business, Pompeu Fabra University, Barcelona, Spain.
Research of S.B. supported by the Spanish Agencia Estatal de Investigaci\'on under project PID2022-138268NB-100 and Grant PID2020-113082GB-I00 funded by MICIU/AEI/10.13039/501100011033.}
\and Francisco Calvillo 
\thanks{LPSM, Sorbonne Université, 4 Place Jussieu, 75005 Paris, France}
\and Gábor Lugosi  
\thanks{ICREA, Pg. Llu\'is Companys 23, 08010 Barcelona, Spain;
Department of Economics and Business, Pompeu Fabra University, Barcelona, Spain; 
Barcelona School of Economics.
Research of G.L. and P.Z. is supported by the Spanish Ministry of Economy and Competitiveness, Grant PGC2018-101643-B-I00 and FEDER, EU.}
\and Piotr Zwiernik
\thanks{Department of Economics and Business, Pompeu Fabra University, Barcelona, Spain; 
Barcelona School of Economics.}
}
\date{}

\begin{document}

\maketitle
\begin{abstract}
We consider the problem of testing properties of graphs underlying high-dimensional graphical models. We adopt the model of \emph{covariance queries} introduced by \cite{lugosi2021learning}. 
We study the case when the underlying graph is a tree. 
The main results of the paper show that, while reconstructing the entire tree may be costly, certain global structural properties can be tested efficiently. 
In particular, we design randomized tests for global structural properties that  use a sub-quadratic number of queries. We develop testing procedures for several fundamental properties, including the number of leaves, the maximum degree, the typical distance, and the diameter of the tree. 
%Our methods rely on random sampling, subtree reconstruction, and concentration inequalities for $U$-statistics and related estimators.
For each property, we obtain explicit query complexity bounds that depend on the target threshold and tolerance parameters. 

%We motivate this work by framing it in the  context of Gaussian/binary probabilistic graphical models on trees, where a simple transformation on correlations maps to the given setting. 
\end{abstract}

\section{Introduction}

Graphical models provide a powerful framework for representing dependencies among random variables through the structure of a graph \citep{lauritzen1996graphical}. Each node corresponds to a random variable, while edges encode direct statistical relationships or conditional dependencies. By making these relationships explicit, graphical models offer a compact and interpretable representation of complex high-dimensional distributions. 

Gaussian graphical models are particularly popular and useful due to their simplicity and mathematical tractability. To define a Gaussian graphical model, let $G=([n],E)$ be a graph with $n$ nodes, $[n]=\{1,\ldots,n\}$, and edge set $E$. Denote by $\mathcal{S}^n_+$ the set of symmetric positive definite $n\times n$ matrices.
In Gaussian graphical models it is assumed that the covariance matrix $\Sigma=\cov(X)$ of an $n$-dimensional Gaussian vector $X=(X_1,\ldots,X_n)$ lies
in the set
\begin{equation*}
M(G)\;=\;\{\Sigma\in \mathcal{S}^n_+: (\Sigma^{-1})_{ij}=0 \mbox{ if }i\neq j\mbox{ and }ij\notin E\}.	
\end{equation*}
for some graph $G$. A key property of Gaussian distributions is the fact that
zero entries of the \emph{inverse} covariance matrix indicate conditional independence, that is,
$$
(\Sigma^{-1})_{ij}=0\quad\Longleftrightarrow\quad  X^{(i)}\indep X^{(j)}|X^{([n]\setminus \{i,j\})}~.
$$

Learning the graph structure underlying probabilistic graphical models from data
is a problem with a long history; see \citet{drton2017structure} for a recent exposition. 
In the standard setup, one has access to independent draws of the random vector $X$.
The amount to data required for learning the
underlying graph is well understood. All methods for learning the graph structure require computing the sample covariance; see, for example, \citet{cai2016estimating,dasarathy2016active}. 
However, in a growing number of applications, the number of variables $n$ is so
large that a computational cost of order $n^2$ becomes
prohibitive. This means that even writing down or storing 
an estimate of the covariance matrix  is not feasible.
In order to remedy this problem,
\citet{lugosi2021learning} introduced 
a ``covariance query'' model, in which the statistician can adaptively query the entries of the covariance matrix
$\Sigma$ whose inverse encodes the underlying graph. The goal of the statistician is to assess properties
of the graph with as few queries as possible.
The main result of  \citet{lugosi2021learning} is that if the underlying graph is a 
bounded-degree tree (or, more generally, has a small tree-width), then
one can learn the underlying structure by using only $O(n\polylog n)$ queries.

In ultra-high dimensional regimes, exact reconstruction of the full graph is often neither feasible nor essential. This motivates a shift from reconstruction to \emph{property testing}, where the goal is to decide whether the underlying tree satisfies a given structural property without learning the entire model. Early work in this direction includes~\citet{neykov2019combinatorial}, who studied testing problems for high-dimensional Gaussian graphical models based on limited samples. More recently,~\citet{devroye2024property} developed a systematic framework for testing whether the underlying graph is a tree using limited distance queries.

The present work builds directly on this testing perspective. While \cite{devroye2024property} focused on testing trees, and small separation numbers, here we focus on trees and study richer structural questions that are specific to tree metrics and their induced graphical models. Our results further clarify which aspects of tree structure can be inferred reliably from a subquadratic number of distance queries, and which necessarily require full reconstruction. 
In particular, one may be interested in simple questions, such as "Does the tree contain large hubs?", 
"Does the tree contasin a long path?", or "How many leaves does the tree have?". The main message of the paper is that such questions can often be answered effectively, based on a much smaller number of queries than what is needed for a full reconstruction of the tree.

In such \emph{tree-structured graphical models}, the absence of cycles leads to strong factorization and identifiability properties. In several standard settings—including Gaussian graphical models, binary Markov random fields, and linear structural equation models—the dependence between any two variables admits a simple path-based representation: correlations (or suitably transformed correlations) between two nodes are equal to the product of edge-level correlations along the unique path connecting them. Equivalently, taking logarithms converts correlations into additive quantities that define a tree metric. This correspondence allows one to translate questions about graphical models on trees into questions about distances on weighted trees.

This observation motivates an alternative, but equivalent, viewpoint: instead of working directly with distributions or graphical model parameters, one may study the underlying tree through access to \emph{tree distance queries}. Each query reveals the distance between a pair of nodes, where distance is understood as the additive weight induced by edge correlations along the path between them. This perspective is classical in phylogenetics and underlies many distance-based reconstruction methods \citep{lake1994reconstructing,semple2003phylogenetics}, but it also provides a natural abstraction for more general graphical models on trees. Throughout this paper, we adopt this distance-based viewpoint as our primary formal framework, while keeping in mind its direct interpretation in terms of graphical models.

Closely related are \emph{latent tree models}, introduced by~\citet{Pearl88} under the name \emph{tree-decomposable distributions} and developed further by \citet{zhang04}. In latent tree models, only a subset of the variables is observed, while the remaining nodes of the tree correspond to unobserved variables. Such models arise naturally in applications including phylogenetics, network tomography, and hierarchical clustering, and are treated comprehensively by~\citet{Zw18}. The reason why latent tree models are important for our work is that subsampling nodes from a fully observed tree model induces a latent tree structure on the subsampled nodes, linking the two frameworks in a natural way.

%The classical \emph{tree reconstruction} problem asks for exact recovery of the underlying tree given access to all pairwise distances or correlations. In the fully observed setting, this problem is well understood and admits efficient solutions, starting with the seminal work of~\citet{chow1968approximating}. When some variables are unobserved, reconstruction becomes more challenging. In this latent setting,~\citet{choi11} proposed computationally efficient algorithms under suitable structural assumptions, establishing consistency guarantees for recovering the hidden tree.

%A complementary line of work studies tree reconstruction from a query-complexity perspective, where access to distances or correlations is costly. In the fully observed case,~\citet{lugosi2021learning} analyzed the number of covariance or distance queries required for exact tree recovery. In the latent case,~\citet{DeLuZw24} showed that exact reconstruction of a latent tree with maximal degree $\Delta$ requires $O(n \Delta \log_\Delta n)$ distance queries, matching the lower bound of~\citet{king03}. In the absence of degree constraints, the query complexity necessarily becomes quadratic in the number of nodes, even when reconstruction is information-theoretically possible.

Throughout the paper we work with an exact distance oracle. This formulation isolates the query-complexity aspect of the testing problems and avoids mixing it with a separate sample-complexity analysis. In applications to Gaussian tree models, distances are estimated from covariances, and noisy covariance queries can be handled by tolerance arguments as in \cite{lugosi2021learning}. Obtaining explicit finite-sample guarantees would require additional assumptions controlling the relevant margins, such as lower bounds on edge lengths and upper bounds on the diameter. Since our results concern query complexity rather than sample complexity, we do not pursue this extension here.

The paper is organized as follows.
In Section \ref{sec: problem setup}, we describe the property-testing framework considered in this paper, and state the main results.
In Sections \ref{section: Testing diameters}, 
\ref{section: Testing maximal degrees},
\ref{sec: testing leaves}, and
\ref{sec: Testing typical distance},
we prove the main results on testing for the diameter, maximal degree, number of leaves, and typical distance, respectively. 
In Section \ref{sec:recover-TX}, we describe the procedure of recovering the subtree spanned by a subset of vertices in a tree. This procedure is a key component of several of the proposed testing procedures, where the subtree spanned by a randomly sampled small subset of vertices proves to be useful.

\section{Problem setup and main results}\label{sec: problem setup}

In several classical families of graphical models defined on trees, pairwise dependencies admit a simple path-based representation. 
In particular, for Gaussian tree models, binary Markov random fields (Ising models), and linear structural equation models on trees, the correlation between two variables is equal to the product of edge-level correlations along the unique path connecting them.
Equivalently, taking minus the logarithm of squared correlations transforms this multiplicative structure into an additive one. Since, for $\rho\in [-1,1]$, $-\log\rho^2\ge 0$, this  yields a metric on the vertices of the tree.

This correspondence allows one to represent such graphical models through an underlying \emph{tree metric}, abstracting away distributional details and retaining only structural information.
We refer to Section~5 of~\cite{DeLuZw24} for a detailed discussion of this metric perspective and additional examples.
Throughout this paper, we adopt this abstraction and work directly with tree metrics, keeping in mind their interpretation in terms of correlations in tree-structured graphical models.

Formally, a \emph{tree} \(T = (V(T), E(T))\) is a connected undirected graph with no cycles.
Between any two vertices \(u, v \in V(T)\), there exists a unique path, denoted by \(\overline{uv}\).

\begin{definition}
    Let \(T=(V(T),E(T))\) be a tree. 
    A map \(d : V(T)\times V(T) \to \R_{\ge 0}\) is called a \emph{tree metric} on \(T\) if and only if there exists an edge-weight assignment \(\omega : E(T)\to \R_{>0}\) such that for all \(u,v\in V(T)\),
    \[
        d(u,v)=
        \begin{cases}
            \sum_{e\in\overline{uv}}\omega(e) & \text{if } u\neq v,\\
            0 & \text{otherwise}.
        \end{cases}
    \]
\end{definition}

Throughout this work, we consider an unknown tree \(T\) on \(n\) vertices, equipped with an unknown tree metric \(d\).
The tree is accessible only through an oracle that returns exact distance values \(d(u,v)\) for queried pairs \(u,v \in V(T)\).
This oracle model corresponds to the population-level setting in which correlations (or transformed correlations) are known exactly.
Our goal is not to reconstruct the full tree, but to decide whether \(T\) satisfies certain global structural properties using as few distance queries as possible. The \emph{query complexity} of a testing procedure is the total number of queries to the oracle 
that the procedure uses.

We now summarize the main contributions of the paper.  
Throughout this paper, both the diameter \(\diam{T}\) and the typical distance \(\ell_{\mathrm{typ}}\) — defined respectively in Sections~\ref{section: Testing diameters} and~\ref{sec: Testing typical distance} — are understood in terms of the \emph{graph distance} on~\(T\), rather than the metric \(d\) returned by the oracle. 
All the testing procedures considered in this paper follow a common paradigm.
They are based on sampling a \emph{small} subset of vertices uniformly at random from the unknown tree \(T\), and querying the oracle only on pairs where at least one element lies in the sample.
The resulting partial information is then used to infer global structural properties of \(T\), without attempting a full reconstruction.

\subsection{Testing the diameter}\label{sec:diam0}

A diameter of the tree is defined as the length (i.e., the number of vertices) of the longest path in it.  Given a tree $T$ and a positive integer $D$, we consider the hypothesis testing problem
\[
H_{\rm diam}: \quad  \diam{T} \ge D~.
\]
We propose a simple test that uses a small number of distance queries.
The idea behind the proposed test is to exploit the fact that long paths in a tree necessarily intersect many randomly sampled nodes. 
Using only oracle queries between sampled vertices, we can determine, for any sampled pair, how many other sampled vertices lie on the unique path connecting them in \(T\).
When \(\diam{T}\) is large, the sample set is likely to contain a pair of vertices whose connecting path includes a significant fraction of the other sampled vertices.

Following this approach, in Section~\ref{section: Testing diameters}  we propose a testing procedure and prove the following theorem.
\begin{theorem}\label{thm: diameter test}
For any $0<\delta,\epsilon<1$ and integer $D\ge 1$, there exists a randomized testing procedure such that:
\begin{itemize}
    \item [(i)] if $H_{\rm diam}$ holds, the algorithm accepts it with probability at least $1 - \epsilon$, 
    \item [(ii)] if $\mathrm{diam}(T) < (1 - \delta)D$, then $H_{\rm diam}$ is rejected with probability at least $1 - \epsilon$.
\end{itemize}
The query complexity of the procedure is
\[
\Theta\!\left(
    \left(
        \frac{n}{D\delta^2}
        \log\!\left(\frac{n^2}{\epsilon}\right)
    \right)^2
\right)~.
\]
\end{theorem}

\noindent
Observe that, for fixed $\epsilon$ and $\delta$, this testing procedure achieves subquadratic query complexity whenever $D/\log n\to \infty$.
Under this condition, Theorem~\ref{thm: diameter test} guarantees the existence of a diameter test that requires substantially fewer queries than those needed to reconstruct the entire tree. 
The efficiency of this approach is particularly notable when testing for large diameters: when $D$ scales linearly with $n$, the test operates with only a polylogarithmic number of queries.

\subsection{Testing the maximal degree}

Denote the maximal degree of a tree \(T=(V(T),E(T))\) by
$\maxdeg{T}$. Given a tree $T$ and an integer $\Delta \ge 1$, we consider the hypothesis testing problem
\[
H_{\rm deg}: \quad \maxdeg{T} \ge \Delta~.
\]
The main challenge in testing the maximal degree is that high-degree vertices may be rare and difficult to detect through uniform sampling alone.
Our approach is based on efficient reconstruction of the tree spanned by a few randomly sampled nodes. The maximal degree in this subtree serves as a proxy for the maximal degree of~$T$.

This approach is developed in detail in Section~\ref{section: Testing maximal degrees}, where we prove the following theorem.

\begin{theorem}\label{thm: max degree test}
For any $0<\delta,\epsilon<1$ and integer $\Delta \ge 1$, there exists a randomized testing algorithm such that:
\begin{itemize}
    \item [(i)] if $H_{\rm deg}$ holds, the algorithm accepts it with probability at least $1-\epsilon$,
    \item [(ii)] if $\maxdeg{T} < (1-\delta)\Delta$, $H_{\rm deg}$ is rejected with probability at least $1-\epsilon$.
\end{itemize}
The total query complexity is
\[
\Theta\!\left(
\frac{n^2}{\Delta\delta^2}
\log\!\left(\frac{n}{\epsilon}\right)
\right)~.
\]
\end{theorem}

\noindent
For fixed $\epsilon$ and $\delta$, this testing procedure achieves subquadratic query complexity whenever $\Delta/\log n \to \infty$.

It is natural to compare the bound on the query complexity with that of the reconstruction algorithm of \cite{DeLuZw24}, which runs in query complexity $O(n\Delta \log n)$. 
That algorithm can be turned into a tester for the maximal degree by simply supplying the threshold $\Delta$ as input: if the procedure succeeds in reconstructing the entire tree, we accept $T$; otherwise, we reject it. 
This reconstruction-based tester outperforms Theorem~\ref{thm: max degree test} whenever $\Delta \le \sqrt{n}$, yielding a query complexity of order $O(n^{3/2}\log n)$. 
For $\Delta \ge \sqrt{n}$, Theorem~\ref{thm: max degree test} becomes more efficient, and in fact also achieves complexity $O(n^{3/2}\log n)$ for fixed $\epsilon$ and $\delta$. 
Consequently, regardless of the threshold~$\Delta$, the maximal degree can always be tested with query complexity
\[
O\!\left(n^{3/2}\log n\right).
\]

\subsection{Testing the number of leaves}

For a tree $T$, denote by
$
\mathcal{L}(T) = \{v \in V(T) : \deg_T(v) = 1\}
$ 
the set of its leaves.
Given an integer $\Lambda \ge 1$, we consider the hypothesis testing problem
\[
H_{\rm leaf}: \quad |\mathcal{L}(T)| \ge \Lambda~.
\]
The proposed strategy for testing the number of leaves is closely related to that used for testing the maximal degree.
We first reconstruct a subtree of $T$ based on a small random sample of vertices.
The leaves of this subtree are taken as candidates for being leaves of the original tree.
A refinement step then verifies, for each candidate, whether it is indeed a leaf of $T$, using additional distance queries.

The details of this procedure are given in Section~\ref{sec: testing leaves}, leading to the following result.
\begin{theorem}\label{thm: leaves test}
For any $0<\delta,\epsilon<1$ and integer $\Lambda \ge 1$, there exists a randomized testing procedure such that:
\begin{itemize}
    \item [(i)] if $H_{\rm leaf}$ holds, the algorithm accepts it with probability at least $1-\epsilon$,
    \item [(ii)] if $|\mathcal{L}(T)| < (1-\delta)\Lambda$, $H_{\rm leaf}$ is rejected with probability at least $1-\epsilon$.
\end{itemize}
The total query complexity is
\[
\Theta\!\left(
\frac{n^2}{\Lambda\delta^2}
\log\!\left(\frac{1}{\epsilon}\right)
\right)~.
\]
\end{theorem}

\noindent
Hence, the number of leaves can be tested with a subquadratic number of queries whenever $\Lambda \gg 1$.
In particular, when testing whether the tree has a linear number of leaves, the query complexity becomes linear in $n$.

\subsection{Testing the typical distance}

The typical distance of a tree $T$ is defined as
\[
\ell_{\mathrm{typ}} = \frac{1}{n^2}\sum_{u,v\in V(T)} \ell(u,v)~,
\]
where $\ell(u,v)$ denotes the number of vertices on the unique path between $u$ and $v$. Given an integer $\ell \ge 1$, we consider the hypothesis testing problem
\[
H_{\rm typ}: \quad \ell_{\mathrm{typ}}(T) \ge \ell~.
\]
To test this hypothesis, we develop two complementary approaches.
The first approach estimates the typical distance by computing exact graph distances between pairs of sampled vertices, which can be achieved by querying distances from each sampled vertex to all other vertices.
This yields a U-statistic estimator of $\ell_{\mathrm{typ}}(T)$.
The second approach is closer in spirit to the diameter test: instead of computing exact distances, we estimate path lengths by counting how many sampled vertices lie on the path between two sampled endpoints.
Averaging over sampled pairs provides an alternative estimator for the typical distance.

By combining these two approaches—depending on the magnitude of the threshold $\ell$—we obtain the following result.
The full analysis is provided in Section~\ref{sec: Testing typical distance}.

\begin{theorem}\label{thm: typ dist}
For any $0<\delta,\epsilon<1$ and integer $\ell \ge 1$, there exists a randomized testing algorithm such that:
\begin{itemize}
    \item [(i)] if $H_{\rm typ}$ holds, the algorithm accepts it with probability at least $1-\epsilon$,
    \item [(ii)] if $\ell_{\mathrm{typ}}(T) < (1-\delta)\ell$, $H_{\rm typ}$ is rejected with probability at least $1-\epsilon$.
\end{itemize}
The total query complexity is
\[
\Theta\left(
n\left(\frac{\diam{T}}{\delta\ell}\right)^2
\log\!\left(\frac{1}{\epsilon}\right)
\min\left\{
1,\,
\frac{n\log(\frac{1}{\epsilon})}{(\delta \ell)^2}
\right\}
+
\frac{1}{\delta^5}
\left(
\frac{n\log(\frac{n^2}{\epsilon})}{\diam{T}}
\right)^2
\right)~.
\]
\end{theorem}

\noindent Observe that the query complexity in Theorem~\ref{thm: typ dist} includes the term
\[
\frac{1}{\delta^5}\left(\frac{n\log(\frac{n^2}{\epsilon})}{\diam{T}}\right)^2,
\]
which does not depend on the prescribed threshold~$\ell$, but only on $n$, $\diam{T}$, $\delta$, and~$\epsilon$.
This term arises from the need to estimate the diameter of~$T$ as an intermediate step in the typical-distance test.
% Specifically, our procedure first constructs a confidence interval for $\diam{T}$ by iterating the test of Theorem~\ref{thm: diameter test}, as formalized in Corollary~\ref{corolary: diameter}. 
If $\diam{T}$ is known in advance, this preliminary step can be omitted.
In that case, the procedure described in Section~\ref{sec: Testing typical distance} tests the typical distance of the tree with total query complexity
\[
\Theta\!\left(
n\left(\frac{\diam{T}}{\delta\ell}\right)^2
\log\!\left(\frac{1}{\epsilon}\right)
\min\left\{
1,\,
\frac{n\log(\frac{1}{\epsilon})}{(\delta \ell)^2}
\right\}
\right)~.
\]
\noindent Moreover, for fixed $\delta$ and $\epsilon$, and assuming either that $\diam{T}$ is known, or that $\diam{T}\gg 1$, this testing procedure achieves subquadratic query complexity whenever
\[
\frac{\diam{T}}{\ell} = o(\sqrt{n})
\qquad\text{or}\qquad
\frac{\diam{T}}{\ell^2} = o(1)~.
\]

\subsection{From testing to estimation}

The testing procedures developed in this paper can be naturally converted into estimation procedures.
This follows a classical principle in statistics: by testing a sequence of nested hypotheses, one can construct confidence intervals and point estimates for the corresponding parameter.
In our setting, this is achieved by iteratively applying the testing algorithms with varying thresholds.

\begin{corollary}\label{corolary: diameter}
    For any  $0<\delta,\epsilon<1$ there exists a randomized algorithm that takes a tree $T$ as input, and returns a real value $D$ such that $$\diam{T}\in [D(1-\delta),D/(1-\delta)]$$ with probability at least $1-\epsilon.$ The procedure runs with total query complexity $$\Theta\left(\frac{1}{\delta^5}\left(\frac{n\log(\frac{n^2}{\epsilon})}{\diam{T}}\right)^2\right)$$
    with probability at least $1-\epsilon~.$
\end{corollary}
\begin{proof}
    Fix $\delta,\epsilon \in (0,1)$. Consider the following iterative procedure: given a tree $T$ on $n$ vertices, initialize $D_0 = n$. Apply the diameter test from Theorem~\ref{thm: diameter test} using the parameter $D_0$. If the test accepts $T$, stop and return $D_0$. If the test rejects $T$, set $D_1 = (1-\delta) D_0$ and repeat an independent test. Continue this process iteratively, updating $D_k = (1-\delta) D_{k-1}$ at each step, until the test is accepted. The returned value then provides the desired approximation of the diameter with probability at least $1-\epsilon$~.
    When the procedure is correct, the total number of iterations $m$ satisfies
    \[
    (1-\delta)^{m+1}n\le \diam{T} \le (1-\delta)^{m-1} n~.
    \]
    Hence, the total number of queries is 
    \begin{align*}
    \Theta\left(
    \sum_{0\le k \le m} \left(
    \frac{n}{D_k\delta^2}\log\left(\frac{n^2}{\epsilon} \right)
    \right)^2
    \right)
    &= 
    \Theta\left(
    \left(
    \frac{1}{\delta^2}\log\left(\frac{n^2}{\epsilon} \right)
    \right)^2\sum_{0\le k \le m} (1-\delta)^{-2k}
    \right) \\
    &= \Theta\left(
    \left(
    \frac{1}{\delta^2}\log\left(\frac{n^2}{\epsilon} \right)
    \right)^2\frac{1}{\delta}(1-\delta)^{-2(m+1)}
    \right) \\
    &= \Theta\left(\frac{1}{\delta^5}\left(\frac{n\log(\frac{n^2}{\epsilon})}{\diam{T}}\right)^2\right)
    \end{align*}
    with probability at least $1-\epsilon~.$
\end{proof}

This iteration scheme applies equally to all of our other tests, allowing us to construct confidence intervals for the maximal degree, the number of leaves, and the typical distance. This yields Corollaries~\ref{corolary: leaves}, \ref{corolary: maxdeg}, and \ref{corolary: typical distance}. 
We omit the proofs of these corollaries, as they follow the same argument as the proof of Corollary~\ref{corolary: diameter}.

\begin{corollary}\label{corolary: maxdeg}
    For any $0<\delta,\epsilon<1$ there exists a randomized algorithm that for any tree $T$, returns a real value $\Delta$ such that $$\maxdeg{T}\in [\Delta(1-\delta),\Delta/(1-\delta)]$$ with probability at least $1-\epsilon.$ The procedure runs with total query complexity 
    $$\Theta\left( \frac{n^2\log\left(\frac{n}{\epsilon}\right)}{\delta^3\maxdeg{T}}\right)~.$$
    with probability at least $1-\epsilon~.$
\end{corollary}

\begin{corollary}\label{corolary: leaves}
    For any $0<\delta,\epsilon<1$ there exists a randomized algorithm that for any tree $T$, returns a real value $\Lambda$ such that $$|\mathcal{L}(T)|\in [\Lambda(1-\delta),\Lambda/(1-\delta)]$$ with probability at least $1-\epsilon.$ The procedure runs with total query complexity 
    $$\Theta\left( \frac{n^2\log\left(\frac{1}{\epsilon}\right)}{\delta^3|\mathcal{L}(T)|}\right)~.$$
    with probability at least $1-\epsilon~.$
\end{corollary}

\begin{corollary}\label{corolary: typical distance}
    For any $0<\delta,\epsilon<1$ there exists a randomized algorithm that for any tree $T$, returns a real value $\ell$ such that $$\ell_{\textrm{typ}}(T)\in [\ell(1-\delta),\ell/(1-\delta)]$$ with probability at least $1-\epsilon.$ The procedure runs with total query complexity 
    $$\Theta\left( \frac{n\log\left(\frac{1}{\epsilon}\right)}{\delta^3}\left( \frac{\diam{T}}{\ell_{\textrm{typ}}(T)}\right)^2
    + \frac{1}{\delta^5}\left(\frac{n\log(\frac{n^2}{\epsilon})}{\diam{T}}\right)^2\right)~.$$
    with probability at least $1-\epsilon~.$
\end{corollary}

\section{Diameter}\label{section: Testing diameters}
In this section, we present a procedure for testing $H_{\rm diam}:\;\diam{T}\geq D$; see Section~\ref{sec:diam0}.  More precisely, for any $u,v \in V(T)$, define
\[
\ell(u,v) = |\{w\in V(T) : w\in\overline{uv}\}|~,
\]
where $\overline{uv}$ denotes the unique path joining $u$ and $v$ in $T$. 
Thus, $\ell(u,v)$ counts the number of \emph{vertices} lying on this path. 
Note that this differs from the standard graph distance, which counts the number of \emph{edges} instead. 
For technical convenience, we work with $\ell(\cdot, \cdot)$, as the two quantities are related by a simple shift by~1.
For any tree $T$, we define the diameter of $T$ as
\[
\diam{T} := \max_{u,v\in V(T)}\ell(u,v)~.
\]

\subsection{Diameter testing procedure}

Fix $0<\delta,\epsilon<1$ and an integer threshold $D\ge1$.
We propose the following procedure.

\medskip
\noindent
\textbf{Path-count statistic.} First, sample $N$ nodes 
$
X=\{X_1,\dots,X_N\}
$
independently and uniformly at random from $V(T)$. Then, for any $u,v\in V(T)$, define
\[
\widetilde{\ell}_X(u,v)
:= \bigl|\{\,1\le i\le N : d(u,v)=d(u,X_i)+d(X_i,v)\,\}\bigr|~.
\]
Note that $d(u,v)=d(u,X_i)+d(X_i,v)$ if and only if $X_i\in \overline{uv}$. Hence, this quantity counts the number of sampled vertices that lie on the unique path between $u$ and $v$ in $T$.
Using only $N^2$ pairwise distance queries among vertices in $X$, compute
\[
d^* \;:=\; \max_{u,v\in X} \widetilde{\ell}_X(u,v).
\]

\medskip
\noindent
\textbf{Decision rule.}
The test accepts $H_{\rm diam}$ if
\begin{equation}\label{eq:rule-diam}
\frac{d^*}{N} \;>\; \frac{D}{n}(1-\delta/2),    
\end{equation}
and rejects otherwise.

\subsection{Proof of Theorem \ref{thm: diameter test}}

We show that the claimed result holds for the procedure described above. Since the samples are uniform, for every $u,v$, $\widetilde{\ell}_X(u,v)$ follows a binomial distribution with parameters
$\left(N,\ell(u,v)/n\right)$.

\paragraph{Step 1: Behavior under $H_{\rm diam}$.}
Assume that $\diam{T}\ge D$.
Then there exists a path $v_1- v_2-\cdots- v_D$ of length $D$.
Note that, 
\begin{equation}
d^*\ge \widetilde{\ell}_X(v_1,v_D)~. \label{eq:dstar-lower}
\end{equation}
The random variable $\widetilde{\ell}_X(v_1,v_D)$ is binomial with parameters $(N,D/n)$.
By Bernstein’s inequality (see, e.g., Section~2.8 in~\cite{BoLuMa03}),
\begin{equation}
\prob{\widetilde{\ell}_X(v_1,v_D)<\frac{ND}{n}(1-\delta/2)}
\;\le\;
\exp\!\left(-\frac{3\delta^2ND}{4n(6+\delta)}\right)~.
\end{equation}
Hence, if
\[
N \ge \frac{4n(6+\delta)}{3D\delta^2}\log\!\left(\frac{1}{\epsilon}\right)~,
\]
then
\[
\prob{\frac{d^*}{N}\ge \frac{D}{n}(1-\delta/2)}\ge 1-\epsilon~.
\]
We conclude that, under $H_{\rm diam}$, the test  \eqref{eq:rule-diam} is correct with probability at least $1-\epsilon$.

\paragraph{Step 2: Behavior under the alternative.}
Assume that $\diam{T}<(1-\delta)D$.
Then for all $u,v\in V(T)$, $\ell(u,v)\le (1-\delta)D$, and therefore
$\widetilde{\ell}_X(u,v)$ is stochastically dominated by a binomial random variable
$\mathcal{B}$ with parameters $(N,(1-\delta)D/n)$. Using a union bound over all pairs $(u,v)$,
\begin{align}
\prob{\frac{d^*}{N}>\frac{D}{n}(1-\delta/2)}
&\le
\prob{\exists\,u,v\in V(T):\widetilde{\ell}_X(u,v)>\frac{ND}{n}(1-\delta/2)} \\
&\le
n^2\prob{\mathcal{B}>\frac{ND}{n}(1-\delta/2)}~.
\end{align}
Applying Bernstein’s inequality yields
\begin{equation}
\prob{\mathcal{B}>\frac{ND}{n}(1-\delta/2)}
\;\le\;
\exp\!\left(-\frac{3\delta^2ND}{4n(6-5\delta)}\right)~.
\end{equation}
Thus, if
\[
N \ge \frac{4n(6-5\delta)}{3D\delta^2}
\log\!\left(\frac{n^2}{\epsilon}\right)~,
\]
then
\[
\prob{\frac{d^*}{N}>\frac{D}{n}(1-\delta/2)}\le \epsilon~.
\]
Combining the two steps, our procedure satisfies the statement of the theorem by choosing
\[
N = \max \left\{
    \frac{4n(6 + \delta)}{3D\delta^2} \log\!\left(\frac{1}{\epsilon}\right), \,
    \frac{4n(6 - 5\delta)}{3D\delta^2} \log\!\left(\frac{n^2}{\epsilon}\right)
\right\}~.
\]
Therefore, the query complexity of the procedure is
\[
N^2 = \Theta\!\left(
    \left(
        \frac{n}{D\delta^2}
        \log\!\left(\frac{n^2}{\epsilon}\right)
    \right)^2
\right)~.
\]

\section{Recovering a subtree over sampled nodes}
\label{sec:recover-TX}

In the proposed testing procedures for the maximum degree and the number of leaves, we need to solve the following problem: Given $X\subseteq V(T)$, recover the subtree $T_X$ spanned on $X$ given all $n|X|$ distances $d(u,v)$ with $u\in X$ and $v\in V(T)$. The \emph{spanned subtree} is obtained by taking the union of all
paths connecting pairs of sampled vertices.
\begin{lemma}\label{lem:TX}
Let $T$ be a weighted tree on vertex set $V$ and let $X\subseteq V$. Given all distances $d(x,v)$ with $x\in X$ and $v\in V$, one can recover the minimal subtree $T_X$ of $T$ spanning $X$, together with all its edge lengths. Moreover, one can recover all distances $d(u,v)$ with $u\in V(T_X)$ and $v\in V$. The query complexity is $n|X|$.
\end{lemma}
 \begin{proof}
     We have $w\in V(T_X)$ if and only if $d(x,x')=d(x,w)+d(w,x')$ for some $x,x'\in X$. This identifies the vertex set $V(T_X)$ of $T_X$. For the edges, note that $uv\in E(T_X)$ if and only if 
     \begin{itemize}
         \item [(i)] $u,v\in \overline{xx'}$ for some $x,x'\in X$,
         \item [(ii)] if $w\in \overline{xx'}$, then $w$ is closer to $x$, or closer to $x'$ than both $u$ and $v$.
     \end{itemize}
     Both conditions easily translate in terms of the queried distances:
     \begin{itemize}
         \item [(i)] $d(x,x')=d(x,u)+d(u,x')=d(x,v)+d(v,x')$ for some $x,x'\in X$
         \item [(ii)] $d(x,x')=d(x,w)+d(w,x')$ implies that 
         $$d(x,w)\leq \min\{d(x,u),d(x,v)\} \quad \text{or} \quad d(w,x')\leq \min\{d(u,x'),d(v,x')\}~.$$
     \end{itemize}
     If $uv\in E(T_X)$ and $u,v\notin X$, we compute $d(u,v)$ without querying it using the formula $d(u,v)=|d(u,x)-d(v,x)|$. 

It remains to recover distances from $V(T_X)$ to vertices outside $V(T_X)$. Fix $v\in V\setminus V(T_X)$. There is a unique vertex $a(v)\in V(T_X)$ at which the path from $v$ to $V(T_X)$ first meets $V(T_X)$. We claim that $a(v)$ is the unique vertex $a\in V(T_X)$ such that $d(x,v)-d(x,a)$ is independent of $x\in X$. If $a=a(v)$, then every path from $x\in X$ to $v$ passes through $a$, and hence
$$
d(x,v)=d(x,a)+d(a,v)~.
$$
Thus $d(x,v)-d(x,a)$ is constantly equal to $d(a(v),v)$. Conversely, suppose $a\neq a(v)$. Let $e$ be the first edge on the path from $a$ to $a(v)$ in $T_X$. Removing $e$ separates $T_X$ into two components, both containing vertices of $X$ by minimality of $T_X$. Choose $x,y\in X$ on opposite sides of $e$, with $x$ on the side containing $a(v)$ and $y$ on the side containing $a$. For $y$, the path from $y$ to $v$ passes through $a$, so
$$
d(y,v)-d(y,a)=d(a,a(v))+d(a(v),v)~.
$$
For $x$, the path from $x$ to $v$ passes through $a(v)$, while the path from $x$ to $a$ also passes through $a(v)$, and hence
$$
d(x,v)-d(x,a)=d(a(v),v)-d(a(v),a)~.
$$
The two quantities differ by $2d(a,a(v))>0$. Hence $d(x,v)-d(x,a)$ is not constant over $X$. This proves the claim. Thus $a(v)$ is identifiable from the queried distances and the already recovered distances inside $T_X$. Once $a(v)$ is known, we compute
$$
d(a(v),v)=d(x,v)-d(x,a(v))
$$
for any $x\in X$; this is known even when $a(v)\notin X$, because $d(x,a(v))$ has already been recovered inside $T_X$. Finally, for every $u\in V(T_X)$,
$$
d(u,v)=d(u,a(v))+d(a(v),v)~,
$$
where $d(u,a(v))$ is known from the recovered subtree $T_X$. This recovers all distances $d(u,v)$ with $u\in V(T_X)$ and $v\in V$. The number of queried distances is exactly $n|X|$.
     % For the second part, if $u,v\in V(T_X)$ then $d(u,v)$ is identified from the first part. If $v\in V(T)\setminus V(T_X)$, let $x\in X$ be the closest to $v$. Then $d(u,v)$ can be computed by the formula $d(u,v)=d(u,x)+d(x,v)$.
 \end{proof}
\paragraph{Relation to distance-based reconstruction.}
The proof of Lemma~\ref{lem:TX} is constructive and directly yields a procedure for recovering the spanned subtree $T_X$.
It is important to emphasize that this task is fundamentally different from classical distance-based tree reconstruction in phylogenetics.

Standard reconstruction methods, such as those described in Semple and Steel 
\citep[Section~7.3]{semple2003phylogenetics} and originating in the work of
\citet{buneman1971recovery}, aim to recover a tree whose \emph{leaf set} is $X$
using only the pairwise distances $\{d(x,x') : x,x'\in X\}$.
These methods recover the unique additive tree metric on $X$, but they do so
only up to suppression of degree-$2$ vertices along paths.
As a result, the output is generally \emph{not} a subgraph of the original tree~$T$.

In contrast, for our testing procedures—most notably for maximal degree and number of leaves—it is crucial to recover the \emph{embedded} subtree $T_X\subseteq T$,
including all intermediate vertices lying on paths between sampled nodes.
Collapsing degree-$2$ vertices along paths would destroy local degree information
and prevent reliable testing of these properties.
This is why we work with a stronger oracle model that provides distances
$d(x,v)$ for all $x\in X$ and $v\in V(T)$.

\paragraph{An explicit recovery procedure.}
Lemma~\ref{lem:TX} shows that $T_X$ can be recovered from the $n|X|$ distance queries
$\{d(x,v): x\in X,\, v\in V(T)\}$ using only path-membership tests of the form
$d(a,b)=d(a,w)+d(w,b)$.
For completeness, we briefly summarize one concrete implementation.

\medskip
\noindent\textbf{Procedure \textsc{RecoverSpannedSubtree}$(X)$.}
Query $d(x,v)$ for all $x\in X$ and $v\in V(T)$.
Construct
\[
V(T_X) \;=\;\Bigl\{\,w\in V(T): \exists\,x,x'\in X
\ \text{such that}\ d(x,x')=d(x,w)+d(w,x')\,\Bigr\}~.
\]
For $u,v\in V(T_X)$, declare $uv\in E(T_X)$ if and only if $u\neq v$ and there is no
$z\in V(T_X)\setminus\{u,v\}$ satisfying $d(u,v)=d(u,z)+d(z,v)$.
The resulting graph $(V(T_X),E(T_X))$ is exactly the spanned subtree $T_X$.

Moreover, once $T_X$ is identified, all distances $d(u,v)$ with
$u\in V(T_X)$ and $v\in V(T)$ can be computed without further oracle queries:
for $v\notin V(T_X)$, letting $x\in X$ be a closest sampled vertex to $v$,
we have $d(u,v)=d(u,x)+d(x,v)$.

\medskip
This recovery step is used only as a subroutine.
Our main results concern testing global properties of $T$ with subquadratic query complexity,
not reconstructing the full tree.

\section{Maximum degree}\label{section: Testing maximal degrees}

In this section, we study the problem of testing the maximum degree of a tree~$T$. Our approach consists of sampling a small subset $X$ of nodes and computing the largest degree in the subtree $T_X$ spanned on this sample.

\subsection{Maximum degree testing procedure}

\noindent\textbf{Degree statistic.} Sample $X=\{X_1,\dots,X_N\}$ uniformly at random without replacement from $V(T)$. Query all $Nn-{N \choose 2}$ distances $d(x,v)$ for $x\in X$ and $v\in V(T)$. Compute the subtree $T_{X}=(V(T_X),E(T_X))$ spanned by $X$ as described in Section~\ref{sec:recover-TX}. For any $v \in V(T_X)$, define $\degree(v)$ to be the number of sampled points that are neighbors of $v$:
\[
\degree(v)=|X\cap\mathcal{N}_T(v)|~.
\]
We set
\[
\degree^* = \max_{v\in V(T_X)} \degree(v)~.
\]

\noindent\textbf{Decision rule.}
The test accepts $H_{\rm deg}$ if and only if
\[
\frac{\degree^*}{N} > \frac{\Delta}{n}(1 - \delta/2)~.
\]

\subsection{Proof of Theorem~\ref{thm: max degree test}}

Fix $0 < \delta,\epsilon < 1$, and let $\Delta \ge 1$. Consider the procedure described above to test $H_{\rm deg}$. We analyze the behavior of this test under the null hypothesis and the alternative.

\paragraph{Step 1: Behavior under $H_{\rm deg}$.}
Assume that there exists a vertex $v\in V(T)$ such that $\deg_T(v)\ge \Delta$.
Since the sample $X$ is drawn uniformly without replacement, the random variable
$|X\cap\mathcal{N}_T(v)|$ follows a hypergeometric distribution with parameters
$(n,N,\deg_T(v))$.
Using Bernstein’s inequality for sampling without replacement (see Theorem~4 in~\cite{Ho63}),
we obtain
\[
\prob{
\frac{|X\cap\mathcal{N}_T(v)|}{N}
<
\frac{\Delta}{n}(1-\delta/2)
}
\;\le\;
\exp\!\left(
-\frac{3\delta^2N\Delta}{4n(6+\delta)}
\right)~.
\]
Hence, if
\[
N \ge
\frac{4n(6+\delta)}{3\delta^2\Delta}
\log\!\left(\frac{1}{\epsilon}\right)~,
\]
then
\[
\prob{
\frac{\degree^*}{N}
\ge
\frac{\Delta}{n}(1-\delta/2)
}
\ge
1-\epsilon~.
\]

\paragraph{Step 2: Behavior under the alternative.}
Assume now that  every vertex in $T$ has degree at most $(1-\delta)\Delta$.
By the discussion above, for each $v\in V(T)$ the random variable $\degree(v)$
is stochastically dominated by a hypergeometric random variable $\mathcal{H}$
with parameters $(n,N,(1-\delta)\Delta)$.
Applying a union bound over all vertices yields
\begin{align*}
\prob{
\frac{\degree^*}{N}
>
\frac{\Delta}{n}(1-\delta/2)}
&\le
n\,\prob{
\mathcal{H}
>
\frac{N\Delta}{n}(1-\delta/2)} \\
&\le
n\exp\!\left(
-\frac{3\delta^2N\Delta}{4n(6-5\delta)}
\right)~.
\end{align*}
Therefore, if
\[
N \ge
\frac{4n(6-5\delta)}{3\delta^2\Delta}
\log\!\left(\frac{n}{\epsilon}\right)~,
\]
we obtain
\[
\prob{
\frac{\degree^*}{N}
>
\frac{\Delta}{n}(1-\delta/2)}
\le
\epsilon~.
\]
This shows that our procedure satisfies the claims of the theorem provided that
\[
N = \max\left\{
\frac{4n(6+\delta)}{3\delta^2\Delta}\log\!\left(\frac{1}{\epsilon}\right),
\quad
\frac{4n(6 - 5\delta)}{3\delta^{2}\Delta} \log\!\left(\frac{n}{\epsilon}\right)
\right\}~.
\]
Moreover, the number of queries needed to construct $T_X$ is $nN$. 
Hence, the query complexity of this test is 
\[
\Theta\!\left( \frac{n^2}{\Delta \delta^2}\log\!\left(\frac{n}{\epsilon}\right) \right)~.
\]

\section{Number of leaves}\label{sec: testing leaves}
To test the hypothesis $H_{\rm leaf}: |\mathcal{L}(T)| \ge \Lambda$, we follow a two-step approach
that parallels the maximal-degree test.
Starting from a random subset of vertices~$X$, 
after querying all $n|X|$ distances between vertices of $X$ and vertices of $V(T)$, we first reconstruct the
subtree $T_X$. The leaves of the subtree $T_X$ form a natural set of candidates for $\mathcal{L}(T)$.
However, partial reconstruction may introduce spurious candidates that are
not true leaves of the original tree.
We refine this candidate set by checking, for each $v\in X$,
whether it is adjacent to exactly one vertex in~$T$.
Since all distances from the vertices in $X$ to vertices in $V(T)$ are queried, this check can be performed by the following simple observation:

\begin{lemma}\label{lem:TX2}
    Consider the same assumptions as in Lemma~\ref{lem:TX}. 
Let $v$ be a leaf of $T_X$. Then $v$ is not a leaf of $T$ if and only if there exists
$u\in V(T)\setminus V(T_X)$ such that the closest point of $T_X$ to $u$ is $v$.
% Piotr: The earlier formulation was actually wrong I think.
    % Then, the tree distance from $u\notin V(T_X)$ to $T_X$ is minimized at $v\in V(T_X)$ if and only if $v$ is not a leaf of $T$.
\end{lemma}

This lemma gives us an explicit way of checking if a leaf of $T_X$ is a leaf of $T$.

\subsection{Number-of-leaves testing procedure}

The proposed procedure is defined as follows. 

\medskip
\noindent\textbf{Test statistic.} Sample $X=\{X_1,\dots,X_N\}$ uniformly at random without replacement from $V(T)$.
Compute the subtree $T_X$ from $X$, and run a test to check which leaves of $T_X$ are also leaves of $T$. Let
\[
\mathcal{L}_X = X \cap \mathcal{L}(T)~,
\]
and set $L^* = |\mathcal{L}_X|$.

\medskip
\noindent\textbf{Decision rule.}  The test accepts $H_{\rm leaf}$ if
\[
\frac{L^*}{N} > \frac{\Lambda}{n}(1-\delta/2)~,
\]
and rejects otherwise.

\subsection{Proof of Theorem~\ref{thm: leaves test}}

    Fix $0<\delta,\epsilon<1$, and let $\Lambda\ge 1$ be an integer. 
    Consider a tree $T$ on $n$ vertices and run the procedure described above to test $H_{\rm leaf}$.     By construction,
    \[
        L^* = |X \cap \mathcal{L}(T)|~.
    \]
    Thus, $L^*$ follows a hypergeometric distribution with parameters $(n,N,|\mathcal{L}(T)|)$.
    We analyze the behavior of this test under the null hypothesis and the alternative.
    
\paragraph{Step 1: Behavior under $H_{\rm leaf}$.} 
    If $|\mathcal{L}(T)| > \Lambda$, Bernstein's inequality for hypergeometric distributions yields
    \begin{equation}\label{eq: leaf inequality 1}
        \mathbb{P}\!\left(\frac{L^*}{N} < \frac{\Lambda}{n}(1-\delta/2)\right)
        \le \exp\!\left(\!-\frac{3\delta^2 N \Lambda}{4n(6+\delta)}\right)~.
    \end{equation}
\paragraph{Step 2: Behavior under the alternative.}    If $|\mathcal{L}(T)| < \Lambda(1-\delta)$, then
    \begin{align}
        \mathbb{P}\!\left(\frac{L^*}{N} > \frac{\Lambda}{n}(1-\delta/2)\right) 
        &= \mathbb{P}\!\left(\frac{L^*}{N} > \frac{\Lambda(1-\delta)}{n}\left(1+\frac{\delta}{2(1-\delta)}\right)\right) \nonumber \\
        &\le \exp\!\left(\!-\frac{3\delta^2 N \Lambda}{4n(6-5\delta)}\right)~. \label{eq: leaf inequality 2}
    \end{align}
    Combining \eqref{eq: leaf inequality 1} and \eqref{eq: leaf inequality 2}, we see that choosing
    \[
        N 
        = \max\left\{
            \frac{4n(6+\delta)}{3\delta^2 \Lambda}\log\!\left(\frac{1}{\epsilon}\right),
            \quad
            \frac{4n(6-5\delta)}{3\delta^2 \Lambda}\log\!\left(\frac{1}{\epsilon}\right)
        \right\}~,
    \]
    guarantees that the test succeeds with probability at least $1-\epsilon$ whenever 
    \[
        |\mathcal{L}(T)| > \Lambda 
        \quad \text{or} \quad 
        |\mathcal{L}(T)| < \Lambda(1-\delta)~.
    \]
    Moreover, the number of queries needed to run this test is $nN$, leading to a query complexity of
\[
    \Theta\!\left( \frac{n^2}{\Lambda\,\delta^2}\log\!\left(\frac{1}{\epsilon}\right)\right)~.
\]

\section{Typical distance}\label{sec: Testing typical distance}

As in Section~\ref{sec: problem setup}, we are given a threshold $\ell$ and a tolerance parameter $\delta$. We aim to design a randomized test for $H_{\rm typ}: \ell_{\rm typ}(T)\geq \ell$. We develop two complementary testing procedures, suited to different regimes of the threshold~$\ell$.
The first approach, presented in Section~\ref{sec: typ distance 1}, computes exact path lengths between sampled vertices and constructs a $U$-statistic centered at $\ell_{\mathrm{typ}}$.
The second approach, developed in Section~\ref{sec: typ dist 2}, replaces exact path lengths by estimators based on sampled path intersections.
Theorem~\ref{thm: typ dist} follows by combining these two procedures and selecting the one with smaller query complexity for the given regime of~$\ell$.

\subsection{A test based on a U-statistic}
\label{sec: typ distance 1}

Let $0<\delta,\epsilon<1~.$ Fix an integer $N \ge 1$, and let $X = \{X_1, \dots, X_N\}$ be a set of $N$ independent vertices sampled uniformly from $T$. 
Observe that by querying the distances $d(x,v)$ for every sampled vertex $x \in X$ and every vertex $v \in V(T)$, one can compute the exact lengths of the paths between any two sampled nodes:
\[
\ell(X_i,X_j)
= \left| \left\{ v\in V(T): d(X_i,X_j)=d(X_i,v)+d(v,X_j)\right\} \right|~.
\]

\noindent Consequently, after performing $nN$ queries, we may define the estimator
\[
\ell_1^* \;:=\; \frac{2}{N(N-1)} \sum_{1\le i<j\le N} \ell (X_i,X_j)~.
\]
Based on $\ell_1^*$, we consider the following test: compute $\ell_1^*$ and \emph{accept} $T$ whenever  
\[
\ell_1^* > \ell(1 - \delta/2)~,
\]
and \emph{reject} otherwise.
If $\diam{T}$ were known, one could choose
\[
N \;\ge\; \left(\frac{4\,\diam{T}}{\delta \ell}\right)^{\!2}
\log\!\left( \frac{1}{\epsilon} \right)~,
\]
and conclude by applying Lemma~\ref{lemma: u-stat} below.  
However, when no prior information on $\diam{T}$ is available, we may apply Corollary~\ref{corolary: diameter} to obtain a random value $D$ satisfying
\[
\diam{T}\in [\,D(1-\delta),\, D/(1-\delta)\,]
\]
with probability at least $1-\epsilon/2$.
We therefore choose
\[
N \;\ge\; 
\left( \frac{4\,D/(1-\delta)}{\delta\ell} \right)^{\!2}
\log\!\left(\frac{2}{\epsilon}\right)~,
\]
sample a fresh set $X$ of $N$ independent uniformly random vertices (independent of the computation of $D$), and then run the test above.  
With probability at least $1-\epsilon/2$, this choice guarantees that
\begin{equation}\label{eq: conditon on N}
    N \;\ge\; \left(\frac{4\,\diam{T}}{\delta\ell}\right)^{\!2}
    \log\!\left(\frac{2}{\epsilon}\right)~.
\end{equation}
Conditioning on the event that \eqref{eq: conditon on N} holds, we conclude by applying Lemma~\ref{lemma: u-stat} that the test succeeds with probability at least $1-\epsilon$ whenever $\ell_{\textrm{typ}} \ge \ell$ or $\ell_{\textrm{typ}} \le \ell(1-\delta)~.$

\medskip
\noindent
Moreover, with probability at least $1-\epsilon$ the number of queries made to compute $D$ is $$\Theta\left(\frac{1}{\delta^5}\left(\frac{n\log(\frac{n^2}{\epsilon})}{\diam{T}}\right)^2\right)$$
and the number of queries made to compute $\ell_1^*$ is 
$$nN=\Theta\left( n\left(\frac{\diam{T}}{\delta\ell}\right)^2\log\!\frac{1}{\epsilon}\right)~,$$
with probability at least $1-\epsilon~.$
Hence, the total number of queries needed for this procedure is
\[
\Theta\left( n\left(\frac{\diam{T}}{\delta\ell}\right)^2\log\!\frac{1}{\epsilon}
+
\frac{1}{\delta^5}\left(\frac{n\log(\frac{n^2}{\epsilon})}{\diam{T}}\right)^2
\right)
\]
with probability at least $1-\epsilon~.$

\begin{lemma}\label{lemma: u-stat}
    Let $\delta,\epsilon>0$ and let $\ell\ge1$ be an integer. If $$N\ge\left(\frac{4\,\diam{T}}{\delta\ell}\right)^2\log\left(\frac{1}{\epsilon}\right)~,$$
    then
    \begin{itemize}
        \item [(i)] $\prob{\ell_1^*\ge \ell(1-\delta/2)}\ge 1-\epsilon~$ when $\ell_{\textrm{typ}}\ge \ell~,$
        \item [(ii)] $\prob{\ell_1^*\ge \ell(1-\delta/2)}\le \epsilon~$ when $\ell_{\textrm{typ}}\le\ell(1-\delta)~.$
    \end{itemize}
\end{lemma}
\begin{proof}
\textbf{Step 1: Behavior under $H_{\rm typ}$.}  Suppose that $\ell_{\mathrm{typ}} \ge \ell$. We have
\begin{equation}\label{eq: ustat 1}
\prob{\ell_1^* \le \ell(1 - \delta/2)} 
\le \prob{\ell_1^* \le \ell_{\mathrm{typ}} - \delta \ell / 2}~.
\end{equation}
Observe that for any $1 \le i < j \le N$, we have $\mathbb{E}[\ell(X_i, X_j)] = \ell_{\mathrm{typ}}$. 
Hence,
\[
\ell_1^* - \ell_{\mathrm{typ}} 
= \frac{2}{N(N - 1)} \sum_{1 \le i < j \le N} \bigl(\ell(X_i, X_j) - \ell_{\mathrm{typ}}\bigr)
\]
is a $U$-statistic of order~2 with kernel 
\[
h:(x, y) \mapsto \ell(x, y) - \ell_{\mathrm{typ}}~.
\]
Moreover, $\|h\|_\infty \le \mathrm{diam}(T)$ and $\mathbb{E}[h(X_1, X_2)] = 0$. 
By Hoeffding's inequality for $U$-statistics (see ~\cite{Ho63}), it follows that
\begin{equation}\label{eq: u-stat}
\prob{\ell_1^* - \ell_{\mathrm{typ}} < - \delta \ell / 2}
\le \exp\!\left(
    -\frac{\lfloor N/2\rfloor (\delta \ell)^2}{8\,\mathrm{diam}(T)^2}
\right)~,
\end{equation}
where $\lfloor N/2\rfloor$ denotes the integer part of $N/2$. 
Hence, choosing
\[
N \ge \left(\frac{4\,\mathrm{diam}(T)}{\delta \ell}\right)^2 
    \log\!\left(\frac{1}{\epsilon}\right)~,
\]
it follows from~\eqref{eq: ustat 1} and~\eqref{eq: u-stat} that
\[
\prob{\ell_1^* \le \ell(1 - \delta/2)} \le \epsilon~.
\]

\medskip
\noindent \textbf{Step 2: Behavior under the alternative.} Now assume that $\ell_{\mathrm{typ}} \le \ell(1 - \delta)$. Then,
\begin{align}
\prob{\ell_1^* \ge \ell(1 - \delta/2)} 
&= \prob{\ell_1^* \ge \ell(1 - \delta) + \delta \ell / 2} \nonumber\\
&\le \prob{\ell_1^* \ge \ell_{\mathrm{typ}} + \delta \ell / 2}~.
\label{eq: u-stat 3}
\end{align}
Applying Hoeffding's inequality for U-statistics again, we obtain
\begin{equation}\label{eq: u-stat hoeffding}
\prob{\ell_1^* - \ell_{\mathrm{typ}} \ge \delta \ell / 2}
\le \exp\!\left(
    -\frac{\lfloor N/2\rfloor (\delta \ell)^2}{8\,\mathrm{diam}(T)^2}
\right),
\end{equation}
which, together with~\eqref{eq: u-stat 3}, implies that by choosing
\[
N \ge \left(\frac{4\,\mathrm{diam}(T)}{\delta \ell}\right)^2 
    \log\!\left(\frac{1}{\epsilon}\right)~,
\]
we have
\[
\prob{\ell_1^* \ge \ell(1 - \delta/2)} \le \epsilon~.
\]
\end{proof}

\subsection{Testing via estimated path length averages}\label{sec: typ dist 2}
Fix $\epsilon> 0 $ and $0<\delta<1~.$ Let $N\ge 1$ be an integer and let $X = \{X_1,\dots,X_N\}$ be a set of $N$ independent vertices drawn uniformly from $T~.$
For any pair of vertices $u, v \in T$, recall the definition from Section~\ref{section: Testing diameters}:
$$\widetilde{\ell}_X(u,v) := \left| \left\{ 1 \le i \le N : d(u,v) = d(u, X_i) + d(X_i, v) \right\} \right|~.$$
Define $$\ell^*_2 =\frac{2}{N(N-1)}\sum_{1\le i<j\le N} \left(\widetilde{\ell}_X(X_i,X_j) -2\right)~,$$ and consider the test that \emph{accepts} $T$ whenever $$\frac{\ell_2^*}{N-2}>\frac{ \ell}{n}(1-\delta/2)~,$$ and \emph{rejects} $T$ otherwise. 
Let $$V_N= \frac{2}{N(N-1)}\sum_{1\le i < j \le N} \left(\frac{1}{N-2}(\widetilde{\ell}_X(X_i,X_j) -2) - \frac{1}{n}\ell(X_i,X_j)\right)~,$$
so that 
\begin{equation}\label{eq: typ distance decomposition}
    \frac{\ell_2^*}{N-2} = \frac{\ell_{\textrm{typ}}}{n} + \frac{1}{n}(\ell_1^*-\ell_{\textrm{typ}}) + V_N~.
\end{equation}

\noindent By using Hoeffding's inequality for U-statistics we have that 
\begin{equation}\label{eq: typ distance path averages 1}
    \prob{|\ell_1^*-\ell_{\textrm{typ}}|\ge \frac{\ell\delta}{4}}\le 2\exp\left(-\frac{\lfloor N/2\rfloor(\ell\delta)^2}{32\diam{T}^2} \right)~,
\end{equation}
and from Lemma \ref{lemma: typ dist bernoulli concentration}, it follows that
\begin{equation}\label{eq: typ distance path averages 2}
    \prob{|V_N|>\frac{\ell\delta}{4n}}\le N^2\exp\left( \frac{-(N-2)(\delta\ell)^2}{8n\cdot\diam{T}+2n(\delta\ell)/3}\right)~.
\end{equation}
Hence, if $\diam{T}$ is known, one can choose $$N=\Theta\left( \frac{n\cdot\diam{T}}{(\delta\ell)^2} \log\frac{1}{\epsilon}\right)~,$$
so that 
\begin{equation}\label{eq: typ distance split proba}
    \prob{|\ell_1^*-\ell_{\textrm{typ}}|\ge \frac{\ell\delta}{4}} + \prob{|V_N|\ge \frac{\ell\delta}{4n}}\le \epsilon~.
\end{equation}
If $\diam{T}$ is unknown, one may follow the same strategy as in Section~\ref{sec: typ distance 1} and apply an independent algorithm from Corollary~\ref{corolary: diameter} to obtain a random variable $D$ such that 
\[
\diam{T}\in [D(1-\delta), D/(1-\delta)]
\]
with probability at least $1-\epsilon/2~.$ One can then choose $$N=\Theta \left(\frac{n\cdot D}{(\delta\ell)^2}\log\frac{2}{\epsilon}\right)~,$$ and inequality \eqref{eq: typ distance split proba} still holds.

\noindent \textbf{Step 1: Behavior under $H_{\rm typ}$.}  Suppose that $\ell_{\textrm{typ}}\ge \ell$. From \eqref{eq: typ distance decomposition} it follows that $$\frac{\ell_2^*}{N-2}\ge \frac{\ell}{n}\left( 1 + \frac{1}{\ell}(\ell_1^*-\ell_{\textrm{typ}})+\frac{n V_N}{\ell}\right)~.$$ Hence,
\begin{equation*}
\prob{\frac{\ell_2^*}{N-2}\le \frac{\ell}{n}(1-\delta/2)}\le \prob{\ell_1^* - \ell_{\textrm{typ}}\le-\frac{\ell \delta}{4}} + \prob{V_N\le -\frac{\ell \delta}{4n}}~.   
\end{equation*}
From \eqref{eq: typ distance split proba}, we have 
$$\prob{\frac{\ell_2^*}{N-2}\le \frac{\ell}{n}(1-\delta/2)} \le \epsilon~.$$

\noindent
\textbf{Step 2: Behavior under the alternative.}  Suppose that $\ell_{\textrm{typ}}\le \ell(1-\delta)$. 
From \eqref{eq: typ distance decomposition} we have 
$$\frac{\ell_2^*}{N-2}\le \frac{\ell(1-\delta)}{n}\left( 1 + \frac{(\ell_1^*-\ell_{\textrm{typ}})}{\ell(1-\delta)} + \frac{nV_N}{\ell(1-\delta)}\right)~.$$ 
Thus, 
\begin{align}
    \prob{\frac{\ell_2^*}{N-2}\ge \frac{\ell}{n}(1-\delta/2)} &\le \prob{1+\frac{\ell_1^*-\ell_{\textrm{typ}}}{\ell(1-\delta)}+\frac{nV_N}{\ell(1-\delta)}\ge 1 + \frac{\delta}{2(1-\delta)}} \nonumber \\
    &\le \prob{\ell_1^*-\ell_{\textrm{typ}} \ge \frac{\ell\delta}{4}} + \prob{V_N\ge \frac{\ell\delta}{4n}}~.\nonumber
\end{align}
From \eqref{eq: typ distance split proba} we conclude that 
$$\prob{\frac{\ell_2^*}{N-2}\ge \frac{\ell}{n}(1-\delta/2)}\le \epsilon~.$$
Hence, this testing algorithm is correct with probability $1-\epsilon$ whenever $\ell_{\textrm{typ}}\ge \ell$ or $\ell_{\textrm{typ}}\le \ell(1-\delta).$ Moreover, the number of queries needed to compute $D$ is
$$\Theta\left(\frac{1}{\delta^5}\left(\frac{n\log(\frac{n^2}{\epsilon})}{\diam{T}}\right)^2\right)~,$$
and the number of queries needed to compute $\ell_2^*$ is 
\[ N^2= 
\Theta\!\left(
\left( \frac{n\cdot \diam{T}}{(\delta\ell)^2}\log\frac{1}{\epsilon} \right)^2
\right)
\]
with probability at least $1-\epsilon~.$ Hence, the total number of queries needed to run this procedure is 
\[
\Theta\!\left(
\left( \frac{n\cdot \diam{T}}{(\delta\ell)^2}\log\frac{1}{\epsilon} \right)^2
\;+\;
\frac{1}{\delta^5}\left(\frac{n\log(\frac{n^2}{\epsilon})}{\diam{T}}\right)^2
\right)
\]
with probability at least $1-\epsilon~.$
\begin{lemma}\label{lemma: typ dist bernoulli concentration}
For any $t>0$, $$\prob{|V_N|>t} \;\le\; N^2\exp \left(\frac{-(N-2)t^2/2}{\diam{T}/n + t/3} \right)~.$$
\end{lemma}

\begin{proof}
Fix $t>0$. By a simple union bound we have
\begin{align}
    \prob{|V_N|>t} \le \frac{N^2}{2}\prob{\left| \widetilde{\ell}_X(X_1,X_2) -2- \frac{N-2}{n}\ell(X_1,X_2)\right|>(N-2)t} ~.\label{eq: typ dist V}
\end{align}
Conditioned on $(X_1,X_2)$, the random variable $\widetilde{\ell}_X(X_1,X_2)-2$ follows a binomial distribution with parameters $N-2$ and $\ell(X_1,X_2)/n$. Bernstein's inequality implies, for any fixed $s>0$,
\begin{align*}
    \mathbf{P}\bigg( \bigg| \widetilde{\ell}_X(X_1,X_2)- 2& - \frac{N-2}{n}\ell(X_1,X_2)\bigg|>s \bigg{|} X_1,X_2\bigg)  \\
    &\le 2\exp \left( \frac{-s^2/2}{\frac{(N-2)\ell(X_1,X_2)}{n}\left(1-\frac{(N-2)\ell(X_1,X_2)}{n}\right)+s/3} \right)~.
\end{align*}
Using the fact that $\ell(X_1,X_2)$ is bounded by $\diam{T}$, it follows that 

\begin{align*}
    \mathbf{P}_\bigg( \bigg| \widetilde{\ell}_X(X_1,X_2)- 2 - \frac{N-2}{n}&\ell(X_1,X_2)\bigg|>s\bigg{|} X_1,X_2 \bigg) \nonumber \\
    &\le 2\exp \left(\frac{-s^2/2}{(N-2)\diam{T}/n + s/3} \right)~.
\end{align*}
Since the right-hand side of the inequality is deterministic, we conclude that 
\begin{equation*}
\prob{\bigg| \widetilde{\ell}_X(X_1,X_2)- 2 - \frac{N-2}{n}\ell(X_1,X_2)\bigg|>s} \le 2\exp \left(\frac{-s^2/2}{\frac{(N-2)\diam{T}}{n} + s/3} \right)~.
\end{equation*}
Together with \eqref{eq: typ dist V}, it follows that
\begin{equation}
    \prob{|V_N|>t} \le N^2\exp \left(\frac{-(N-2)t^2/2}{\diam{T}/n + t/3} \right)~,
\end{equation}
which concludes the proof.

\end{proof}

\bibliographystyle{plainnat}
\bibliography{biblio.bib}
\end{document}